\documentclass{article}

\usepackage{mdframed}

\usepackage[preprint]{neurips_2023_otml}

\usepackage{amsmath,amsfonts,bm}

\def\eqref#1{equation~\ref{#1}}

\def\1{\bm{1}}

\def\vh{{\bm{h}}}

\def\vp{{\bm{p}}}

\def\vx{{\bm{x}}}

\def\mC{{\bm{C}}}
\def\mD{{\bm{D}}}
\def\mE{{\bm{E}}}
\def\mF{{\bm{F}}}

\def\mH{{\bm{H}}}
\def\mI{{\bm{I}}}

\def\mK{{\bm{K}}}
\def\mL{{\bm{L}}}
\def\mM{{\bm{M}}}

\def\mP{{\bm{P}}}

\def\mS{{\bm{S}}}
\def\mT{{\bm{T}}}
\def\mU{{\bm{U}}}
\def\mV{{\bm{V}}}
\def\mW{{\bm{W}}}
\def\mX{{\bm{X}}}

\def\mZ{{\bm{Z}}}

\DeclareMathAlphabet{\mathsfit}{\encodingdefault}{\sfdefault}{m}{sl}
\SetMathAlphabet{\mathsfit}{bold}{\encodingdefault}{\sfdefault}{bx}{n}

\newcommand{\R}{\mathbb{R}}

\DeclareMathOperator{\Tr}{Tr}

\usepackage[utf8]{inputenc} \usepackage[T1]{fontenc}    \usepackage{hyperref}
\usepackage{url}
\usepackage{algorithm}
\usepackage{algorithmic}

\usepackage{amsmath}
\usepackage{amssymb}
\usepackage{mathtools}
\usepackage{amsthm}
\usepackage[capitalize,noabbrev]{cleveref}
\usepackage{dsfont}

\usepackage[dvipsnames]{xcolor}         \usepackage{microtype}
\usepackage{graphicx}
\usepackage{subfigure}
\usepackage{booktabs} \usepackage{caption}
\usepackage{wrapfig}

\theoremstyle{plain}
\newtheorem{theorem}{Theorem}
\newtheorem{corollary}{Corollary}
\newtheorem{lemma}{Lemma}

\theoremstyle{definition}

\title{Interpolating between Clustering and Dimensionality Reduction with Gromov-Wasserstein}

\newcommand{\diag}{\operatorname{diag}}
\newcommand{\integ}[1]{{[\![#1]\!]}}

\newcommand{\GW}{\operatorname{GW}}

\newcommand{\srGW}{\operatorname{srGW}}
\newcommand{\srFGW}{\operatorname{srFGW}}

\newcommand{\scalar}[2]{\langle #1 , #2 \rangle}
\newcommand{\vectorize}{\operatorname{vec}}

\author{  Hugues Van Assel$^*$ \\
  ENS de Lyon, CNRS \\
  UMPA UMR 5669 \\
  \texttt{hugues.van\textunderscore assel@ens-lyon.fr} \\
    \And
  Cédric Vincent-Cuaz$^*$ \\
  EPFL, Lausanne  \\ 
  LTS4 \\
  \texttt{cedric.vincent-cuaz@epfl.ch} \\
    \And
  Titouan Vayer \\
  Univ. Lyon, ENS de Lyon, UCBL, CNRS, Inria \\
  LIP UMR 5668 \\
  \texttt{titouan.vayer@inria.fr} \\
  \And
  Rémi Flamary \\
  École polytechnique, IP Paris, CNRS \\
  CMAP UMR 7641 \\
  \texttt{remi.flamary@polytechnique.edu} \\
  \And
  Nicolas Courty \\
  Université Bretagne Sud, CNRS \\
  IRISA UMR 6074 \\
  \texttt{nicolas.courty@irisa.fr} \\
}

\begin{document}

\maketitle

\def\thefootnote{*}\footnotetext{Equal contribution.}\def\thefootnote{\arabic{footnote}}

\begin{abstract}
    We present a versatile adaptation of existing dimensionality reduction (DR) objectives, enabling the simultaneous reduction of both sample and feature sizes.
    Correspondances between input and embedding samples are computed through a semi-relaxed Gromov-Wasserstein optimal transport (OT) problem.
    When the embedding sample size matches that of the input, our model recovers classical popular DR models. 
    When the embedding's dimensionality is unconstrained, we show that the OT plan delivers a competitive hard clustering. 
    We emphasize the importance of intermediate stages that blend DR and clustering for summarizing real data and apply our method to visualize datasets of images.
\end{abstract}

\section{Introduction}

Summarizing the information carried by a dataset in an unsupervised way is of utmost importance in modern machine learning pipelines \cite{donoho2000high}. Smaller representations of data offer numerous advantages, including improved pattern and structure recognition, as well as faster processing for downstream tasks \cite{mendible2020dimensionality, cantini2021benchmarking, pochet2004systematic}. To construct such representations, one can either reduce the sample size by aggregating points together (referred to as \emph{clustering}) or reduce the feature dimensionality \textit{i.e.}\ performing \emph{dimensionality reduction} (DR). While both tasks are actively studied topics, very few works have proposed a consistent model to simultaneously perform clustering and DR.

\textbf{Contributions.} In this work, we provide a new framework for joint clustering and DR. The goal is to obtain a reduced representation in \emph{both samples and features} \textit{i.e.}\ a transformation $\mX \in \R^{N \times p}$ to $\mZ \in \R^{n \times d}$ where $n < N$ (clustering) and $d < p$ (DR). Doing so, we ensure that the low-dimensional embeddings align well with the class labels determined during clustering.
In \cref{sec:background}, we frame classical DR methods as minimizing a discrepancy between two aligned affinity matrices: $\mC_X$ defining the dependencies among high-dimensional samples and $\mC_Z$ focusing on low-dimensional ones. 
We then propose to augment this general objective using the Gromov-Wasserstein (GW) framework to enable matching affinities of different dimensions. 
When $\mC_Z$ has fewer nodes than $\mC_X$, computing a GW transport plan naturally amounts to a clustering of the input samples, aggregating them into \emph{prototypes}.
Therefore this model leads to a principled objective for simultaneously learning low-dimensional prototypes and their assignments to input samples.
We show in \Cref{theo:srgw_bary_concavity} that, in the context of PSD matrices used in existing DR approaches, the assignments provide a hard clustering of the input samples.
We discuss key properties advocating for the use of this clustering in \Cref{sec:background} before introducing our model in \Cref{sec:model_GWDR} and applying it to real data in \Cref{sec:exps}.

\section{Generalization of Dimension Reduction via Graph Matching}\label{sec:background}

\textbf{Unified view of Dimensionality Reduction.}
Let $\mX = (\vx_1, ..., \vx_N) ^\top \in \R^{N \times p}$ be an input dataset of interest. DR methods focus on constructing a low-dimensional representation or \emph{embedding} $\mZ \in \R^{N \times d}$, where $d$ is smaller than $p$. The latter should preserve a prescribed geometry for the dataset usually encoded via a pairwise similarity matrix $\mC_{X}$. To this end, most popular DR methods (\textit{e.g.}\ kernel PCA \cite{scholkopf1997kernel}, MDS \cite{torgerson1952multidimensional}, Laplacian eigenmaps \cite{belkin2003laplacian}, SNE-like methods \cite{hinton2002stochastic}) optimize $\mZ$ such that its similarity matrix $\mC_Z$ matches $\mC_X$ in accordance with the following objective:
\begin{align}\label{eq:DR_criterion}
	\mathcal{J}_L(\mC_X, \mC_Z) \coloneqq \sum_{(i,j) \in \integ{N}^2} L([\mC_X]_{ij}, [\mC_Z]_{ij})
\end{align}
where $L:\R \times \R \rightarrow \R_+$ is typically the quadratic loss $L_2(x,y) \coloneqq (x - y)^2$ or the generalized Kullback-Leibler divergence $L_{\mathrm{KL}}(x,y) \coloneqq x \log (x/y) - x + y$. As detailed in \Cref{sec:DR_methods}, the definitions of $\mC_X$ and $\mC_Z$ as well as $L$ are what differentiate each method.
Note that these objectives can be derived from a common Markov random field model with various graph priors
\cite{van2022probabilistic}.
The unified objective \cref{eq:DR_criterion} can also be seen as a trivial instance of graph matching where both graph structures $\mC_Z$ and $\mC_X$ are designed so that their nodes are aligned. To promote clustering from this objective, one can enforce $\mC_{Z}$ to have fewer nodes than $\mC_{X}$ $(n< N)$ and seek for meaningful structural correspondences between the nodes of both graphs.

\textbf{Gromov-Wasserstein framework.}
Interestingly, the Optimal Transport (OT, \cite{villani2009optimal,peyre2019computational}) literature provides a way to do so with the Gromov-Wasserstein discrepancy (GW, \cite{memoli2011gromov,sturm2012space, peyre2016gromov, chowdhury2019gromov}). In this context, nodes are endowed with probability weights $\overline{\vh}_X \in \Sigma_N$ and $\overline{\vh}_Z \in \Sigma_n$ encoding their relative importance. GW then computes a soft-assignment matrix between the nodes of the two graphs $(\mC_X, \vh_X)$ and $(\mC_Z, \vh_Z)$,  as well as a notion of dissimilarity between them reading as:
\begin{equation}\label{eq:gw_pb}
	\mathrm{GW}_L(\mC_X, \vh_X, \mC_Z, \vh_Z ) \coloneqq \min_{\mT \in \mathcal{U}(\vh_X,\vh_Z)} \:  \sum_{(i,j) \in \integ{N}^2} \sum_{(k,l) \in \integ{n}^2} L([\mC_X]_{ij}, [\mC_Z]_{kl}) T_{ik} T_{jl}
\end{equation}
where $\mathcal{U}(\vh_X, \vh_Z) = \left\{ \mT \in \R_+^{N \times n} | \mT \bm{1}_n = \vh_X, \mT^\top \bm{1}_N = \vh_Z \right\}$. An optimal coupling $\mT^*$ acts as a soft matching of the nodes, which tends to associate pairs of nodes that have similar pairwise relations in $\mC_X$ and $\mC_Z$ respectively. These properties are clear benefits for many ML tasks such as alignments of diverse structured objects \cite{ solomon2016entropic,alvarez2018gromov, xu2019gromov, demetci2020gromov, bonet2021subspace}, (co-)clustering  \cite{peyre2016gromov, titouan2020co}, graph representation learning \cite{xu2020gromov, vincent2021online, liu2022robust, vincent2022template, pmlr-v202-zeng23c} and partitioning \cite{xu2019scalable, chowdhury2021generalized}. The latter is in line with our objectives as it focuses on the design of a target graph $(\overline{\mC}, \overline{\vh})$, so that the OT resulting from $\GW(\mC_X, \vh_X, \overline{\mC}, \overline{\vh})$ provides a most significant clustering of the nodes in $(\mC_X, \vh_X)$.
A first axiom consisted in fixing $\overline{\mC}$ and optimizing its nodes' relative importance $\overline{\vh}$ modeling cluster proportions \cite{vincent2021semi}.
This problem is efficiently tackled using the semi-relaxed GW divergence (srGW) which interest boils down to minimizing the GW loss in \cref{eq:gw_pb} over $\mathcal{U}_n(\vh_X) = \left\{\mT \in \R_{+}^{N \times n}| \mT \bm{1}_n = \vh_X\right\}$. We argue that a better approach consists of also learning the target structure so that its entries would describe connectivity between clusters allowing a sharper graph partitioning. Which leads to the following optimization problem:
\begin{equation}\label{eq:srgw_bary} \tag{srGWB}
	\min_{\overline{\mC} \in \R^{n \times n}} \: \srGW_L(\mC_X, \vh_X, \overline{\mC}) \Leftrightarrow \min_{\overline{\mC} \in \R^{n \times n} , \overline{\vh} \in \Sigma_n} \GW_L(\mC_X, \vh_X, \overline{\mC}, \overline{\vh}) \:. 
\end{equation}
This amounts to searching for the closest graph $(\overline{\mC}, \overline{\vh})$ of size $n$ to the input graph $(\mC_X, \vh_X)$ in the GW sense. As such, it is a specific instance of srGW barycenter over a single input graph \cite{vincent2021semi}. We next study whether \ref{eq:srgw_bary} admits OT which are actual membership matrices (with a single non null value per row) achieving hard clusterings of the nodes of $\mC_X$.
\setcounter{theorem}{0}
\begin{theorem}\label{theo:srgw_bary_concavity}
		Let $\mC_X \in \R^{N \times N}$ and $\vh_X \in \Sigma_N^*$ a vector in the probability simplex. If $g(\mU)= \vectorize(\mU)^\top \left(\mC_X \otimes_K \mC_X\right) \vectorize(\mU)$
				is \underline{convex} on $\mathcal{U}(\vh_X, \vh_X)$ , then \ref{eq:srgw_bary} with $L=L_2$ admits scaled membership matrices as optimum. \end{theorem}
The sufficient condition in Theorem \ref{theo:srgw_bary_concavity} is satisfied for existing DR methods (\Cref{sec:DR_methods}), \emph{e.g} when $\mC_X$ is PSD (or NSD). In this setting, this result completes the analysis of  \cite{chen2023gromov} establishing that \ref{eq:srgw_bary} constrained to membership matrices as OT is a SOTA graph coarsening method for spectrum preservation, equivalent to a weighted kernel K-means \cite{dhillon2004kernel, dhillon2007weighted}.  Following \cite[equation 6]{vayer2018optimal},we can also see that $g$ is convex whenever the GW problem from a graph to itself is concave. Hence \cite[Proposition 2]{redko2020co} also extends our analysis to squared Euclidean distance matrices. A corollary of \Cref{theo:srgw_bary_concavity} establishes an analog result when
$\overline{\vh}$ is not optimized (\Cref{sec:srGW_concavity}).

		% \end{align}

\section{Joint Clustering and Dimensionality Reduction}\label{sec:model_GWDR}

\textbf{Dimensionality reduction with Gromov-Wasserstein.}
In light of the  results presented above on the clustering abilities of srGW, we introduce a versatile algorithm for joint clustering and dimensionality reduction.
Our method amounts to replacing the usual DR objective \cref{eq:DR_criterion} by a srGW loss \cref{eq:gw_pb} thus allowing to reduce the sample size.  Namely, we learn embeddings $\mZ$ that parametrize a structure $\mC_Z$ induced by the underlying DR method as follows:
\begin{equation}\label{eq:GW_DR_pb}\tag{GW-DR}
	\min_{\mZ \in \R^{n \times d}} \mathrm{srGW}_L(\mC_X, \vh_X, \mC_Z) \:.
\end{equation}
The embeddings $\mZ$ then act as low-dimensional prototypical representations of input samples, whose learned relative importance $\vh_Z$ accommodates clusters or substructures of varying proportions in $\mX$. When 
$\mC_Z = \mZ \mZ^\top$ mimicking e.g PCA (\Cref{sec:DR_methods}),  \ref{eq:GW_DR_pb} boils down to a srGW barycenter problem constrained to have at most rank $d$ which coincides with \ref{eq:srgw_bary} if $d \geq n$. These relations and \Cref{theo:srgw_bary_concavity} allow us to expect OT solutions close to providing a hard-clustering of $\mX$. Finally, we emphasize that the GW framework does not take into account input samples and embeddings explicitly, but only implicitly through their pairwise similarity matrices $\mC_X$ and $\mC_Z$. To readily incorporate the feature information of $\mX$ in \ref{eq:GW_DR_pb}, one can adopt the Fused GW framework \cite{pmlr-v97-titouan19a} that interpolates linearly, via a hyperparameter $\alpha \in [0,1]$, between our objective and a linear OT cost that matches samples $\mX \in \R^{N \times p}$ and a learned feature matrix $\overline{\mF} \in \R^{n \times p}$. The latter essentially reduces to a concave problem, wherein the goal is to achieve K-means clustering on $\mX$ \cite{canas2012learning}, hence acting as a concave regularization of \ref{eq:GW_DR_pb} (see details in \cref{sec:FGW_DR}).

\textbf{Computation.}
\ref{eq:GW_DR_pb} is a non-convex problem that we propose to tackle using a Block Coordinate Descent algorithm (BCD, \cite{tseng2001convergence}) guaranteed to converge to local optimum \cite{grippo2000convergence, Lyu2023bmm}. The BCD alternates between \emph{i)} solving for a srGW problem given $\mZ$ using the Conditional Gradient solver in \cite{vincent2021semi} extended to support $L_{KL}$; \emph{ii)} optimizing $\mZ$ for a fixed OT using gradient descent with adaptive learning rates \cite{kingma2014adam}. Each update is achieved in $\mathcal{O}(nN^2 + n^2N)$ operations. 
\textbf{Related work.}
The closest to our work is the COOT-clustering approach proposed in \cite{redko2020co} that estimates simultaneously a clustering of samples and variables using the  CO-Optimal Transport problem.
The key difference is that we leverage the affinity matrices and kernels of existing DR methods instead of aligning the features.
Other approaches such as \cite{liu2022joint} involve modelling latent variables with mixture distributions. Note that none of the previously proposed methods can easily adapt to the mechanisms of existing DR methods like \cref{eq:GW_DR_pb}.

% The above problem ensures that the embeddings' position is consistent with the prototypes computed using a larger scale of dependency.

\section{Experiments}\label{sec:exps}

\begin{wraptable}[6]{R}{5cm}
    \vspace*{-0.4cm}
	\centering
	\caption{ARI (\%) clustering scores.}\label{tab:GWB_clustering}
	\scalebox{0.9}{
		\begin{tabular}{|c||c|c|} \hline
			& $\mathrm{srGWI}$ & $\mathrm{srGWB}$ \\ \hline \hline
			MNIST & 29.7(1.9) & \textbf{32.6}(\textbf{1.8})\\ \hline
			F-MNIST & 26.1(0.0) & \textbf{39.5}(\textbf{0.3}) \\ \hline
			COIL & 18.1(0.2) & \textbf{51.0}(\textbf{1.7}) \\ \hline
		\end{tabular}}
\end{wraptable}
In this section, we showcase the relevance of our approach on popular image datasets: COIL-20 \cite{nene1996columbia}, MNIST and fashion-MNIST \cite{xiao2017fashion}. 
Results are averages and standard deviations, computed over 5 runs with different random seeds. Details about evaluation metrics and datasets are provided in \Cref{sec:appendix_exps}.
Throughout this section, we set $\vh_X$ as uniform.
In what follows, for any existing DR method, we refer to its gromovized version by appending the prefix "GW" to the method name \textit{e.g.}\ GW-PCA.

\textbf{Clustering.} 
We first evaluate the clustering abilities of srGW barycenters (\ref{eq:srgw_bary}) and their vanilla counterpart with fixed structure $\mI_n$ used in the graph partitioning literature (srGWI, \cite{vincent2021semi}). For both, $\mC_X$ is taken as the MDS kernel (see \Cref{sec:DR_methods}). Clustering performances measured by means of ARI are reported in \Cref{tab:GWB_clustering} and show the superiority of srGWB.

\begin{wraptable}[6]{R}{8.1cm}
    \vspace*{-0.4cm}
    \centering
    \caption{Homogeneity ($\times 100$) scores for GW-tSNEkhorn.}
    \label{tab:homogeneity_scores}
    \scalebox{0.9}{
	\begin{tabular}{|l|c|c|c|c|} \hline
    & $n=10$ & $n=50$ & $n=100$ & $n=200$ \\
    \hline
    \hline
    MNIST & $49.2(1.5)$ & $76.8(1.1)$ & $80.8(0.6)$ & $83.8(0.8)$ \\ \hline
    F-MNIST & $56.0(2.4)$ & $68.9(0.7)$ & $69.8(1.4)$ & $71.9(1.6)$ \\ \hline
    COIL & $55.8(1.2)$ & $77.9(3.2)$ & $82.2 (3.6)$ & $85.3(2.9)$ \\ \hline
    \end{tabular}}
\end{wraptable}

\textbf{Joint Clustering and Dimensionality Reduction.}
In \Cref{fig:visu_gwdr}, we display the prototypes produced by GW-tSNEkhorn (robust version of tSNE presented in \cite{van2023snekhorn}) for various $n$. We used fused srGW \cite{vayer2018optimal} with $\alpha=0.5$ as it naturally produces prototypes in input space (as Wasserstein barycenters of images) that can be visualized.
They show the relatively effective purity of the prototypes confirmed by the homogeneity scores displayed in \Cref{tab:homogeneity_scores} for various $n$. 
Recall that $n$ only provides an upper bound of the number of prototypes as the semi-relaxed OT problem permits the flexibility to discard unnecessary prototypes. 
The latter scores compute to which extent prototypes contain samples of the same label. It's reasonable to note that as the value of $n$ increases, the consistency or similarity among the prototypes also increases.
\begin{figure*}[t]
	\begin{center}
		\centerline{\includegraphics[width=0.84\columnwidth]{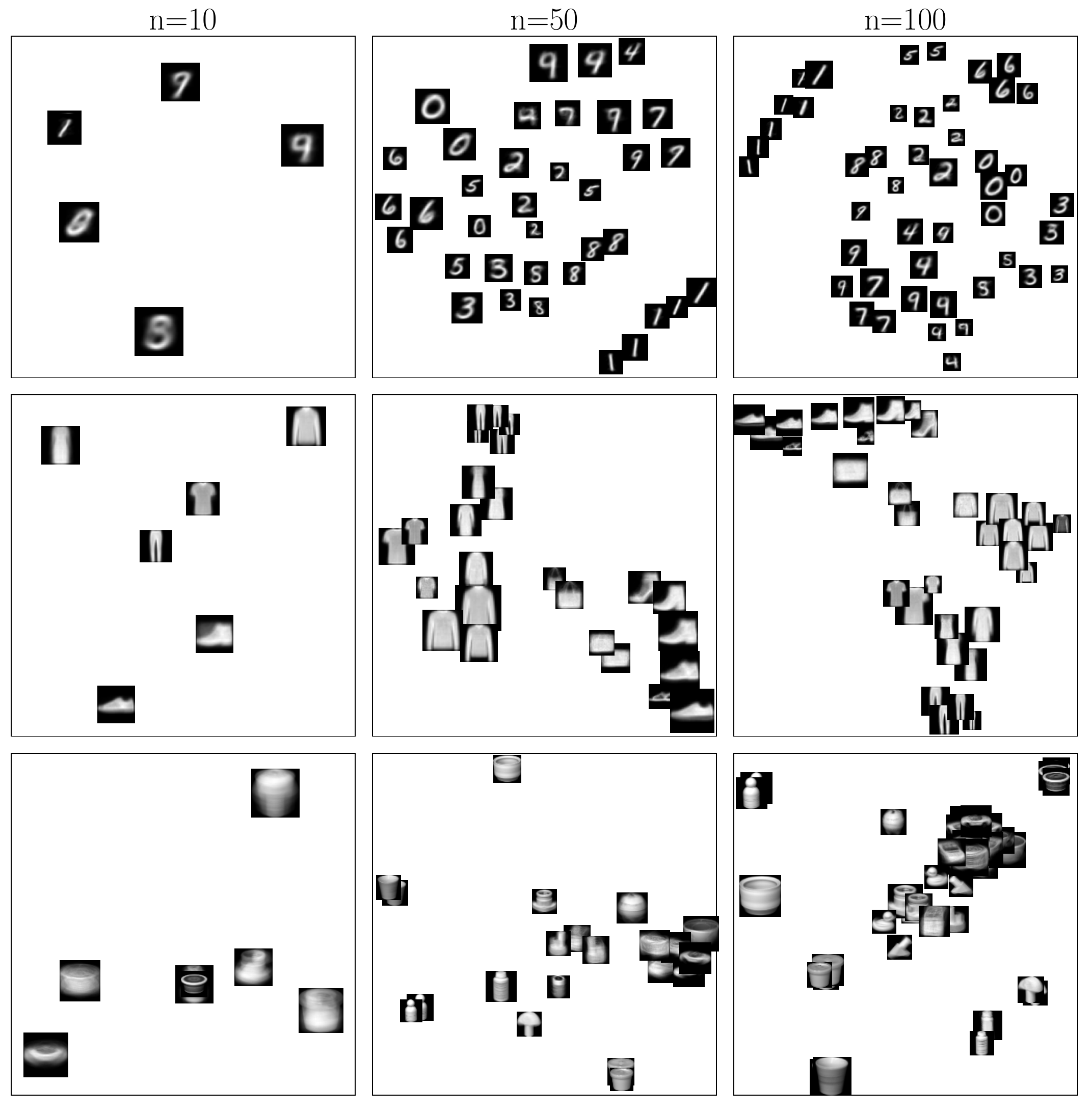}}
		\caption{2D Embeddings of GW-tSNEkhorn applied to MNIST (top), Fashion-MNIST (middle) and COIL (bottom) with various $n$. The perplexity is set to  $\xi=50$ for all experiments. Images for prototypes are computed as Wasserstein barycenters of the associated input images. Their areas are proportional to $\vh_Z$.}
		\label{fig:visu_gwdr}
	\end{center}
	\vspace{-0.8cm}
\end{figure*}

\begin{wraptable}[6]{R}{7.15cm}
    \vspace{-0.4cm}
    \centering
    \caption{Best fused GW parameter $\alpha$ for $n=50$.}\label{tab:alpha_scores}
    \scalebox{0.9}{
        \begin{tabular}{|c||c|c|c|} \hline
            &$\alpha^*$ & Homogeneity & Silhouette  \\ \hline \hline
            MNIST & $0.9997$ & $74.7(0.2)$ & $16.4(5.6)$  \\ \hline
            F-MNIST &  $1$ & $61.5(1.6)$ & $16.3(3.6)$  \\ \hline
            COIL & $0.999999$ & $87.51(0.1)$ & $42.1(5.6)$  \\ \hline
    \end{tabular}}
\end{wraptable}
\textbf{Should clustering depend on embeddings?}
Choosing the fused GW hyperparameter as $\alpha \to 0$ would result in the
clustering ignoring the current positions of embeddings and only leveraging
information about the input $\mX$ (pure clustering). To determine whether this can be beneficial,
we performed a grid search over different values of $\alpha$ (details in
\cref{sec:appendix_exps}). We selected the value $\alpha^*$ that maximizes the
sum of homogeneity and silhouette scores \cite{rousseeuw1987silhouettes}. The
latter is computed based on a ground truth taken as the most represented input
label in the associated prototype. Thus it gives a quantitative metric to
properly evaluate the prototypes' relative positions.
Best scores and their respective $\alpha^*$ are reported in \Cref{tab:alpha_scores} for GW-tSNEkhorn. These illustrate the significance of embedding-dependent clustering to ensure that the embeddings display a meaningful structure, as all $\alpha^*$ are greater than 0.

\section{Concluding Remarks}

We believe that the versatility of our approach will enable applications beyond data visualization.
For instance, the formalism associated with (sr)GW barycenters naturally allows us to consider multiple affinity matrices as inputs. In this context, popular open challenges relate to the multi-scale and multi-view dimensionality reduction problems. We envision to thoroughly investigate the latter both empirically and theoretically, building on Theorem 1 which may also conduct to new discoveries for the GW-based (multi) graph coarsening or dictionary learning.

\bibliographystyle{plain}

\newpage
\appendix

\section{Framing Dimensionality Reduction as Graph Matching}\label{sec:DR_methods}

In this section, we provide a unified view of the most popular DR methods with the following objective, where $L:\R \times \R \rightarrow \R_+$ is the loss function,
\begin{align}
	\mathcal{J}_L(\mC_X, \mC_Z) \coloneqq \sum_{(i,j) \in \integ{N}^2} L([\mC_X]_{ij}, [\mC_Z]_{ij}) \:.
\end{align}

\paragraph{Kernel PCA, MDS and Isomap.} Let us consider $\mK_X = (k(\vx_i, \vx_j))_{ij} \in \mathcal{S}^N_+$ a kernel matrix over the input data $\mX$. Denoting $R_d \coloneqq \{ \mC \in \mathcal{S}_+^{N} \: \text{s.t.} \: \mathrm{rk}(\mC) \leq d \}$ the set of rank at most $d$ PSD matrices, kernel PCA \cite{scholkopf1997kernel} computes $\mS_Z = \mathrm{Proj}^F_{R_d}(\mK_X)$. Since $\mS_Z \in R_d$, we have the existence of $\mZ \in \R^{N \times d}$ such that $\mS_Z = \mZ \mZ^\top$ (sample covariance of $\mZ$). In view of this property, the kernel PCA problem reads
\begin{align}\label{eq:kernel_PCA}
	\min_{\mZ \in \R^{n \times d}} \mathcal{J}_{L_2}(\mK_X, \mS_Z) \:.
	\tag{PCA}
\end{align}
Note that traditional PCA simply amounts to choosing $\mK_X = \mX \mX^\top$ in the above problem. Multidimensional scaling (MDS) \cite{torgerson1952multidimensional} can be easily derived from a slight variation of \ref{eq:kernel_PCA}. Define $\mD_X = - \bm{H}_N \bm{E}_X \bm{H}_N$ with $[\bm{E}_X]_{ij} = \|\vx_{i} - \vx_{j}\|_2^2$ and where $\mH_N = \mI_N - N^{-1} \mathbf{1}_N \mathbf{1}_N^\top$ is the centering matrix. Since $\mE_X$ is a squared Euclidean distance matrix, it results that $\mD_X \in \mathcal{S}^N_+$ \cite{mardia1979multivariate}. Classical MDS then amounts to minimizing the following strain.
\begin{align}\label{eq:MDS}
	\min_{\mZ \in \R^{n \times d}} \mathcal{J}_{L_2}(\mD_X, \mS_Z) \:.
	\tag{MDS}
\end{align}

\paragraph{Laplacian Eigenmaps.}
Let $\mW_X \in \mathcal{S}^N$ be a similarity graph (\textit{e.g.}\ neighborhood graph) built from $\mX$. We define its graph Laplacian as $\mL_X = \diag(\mW_X \bm{1}) - \mW_X$ such that $\mL_X \in \mathcal{S}_+^N$ \cite{chung1997spectral}.
Laplacian eigenmaps \cite{belkin2003laplacian} boils down to the following objective 
\begin{align}\label{eq:SNE}
	\max_{\mZ \in \mathrm{St}(n,d)} \mathcal{J}_{L_2}(\mL_X, \mS_Z)
	\tag{LE}
\end{align}
where $\mathrm{St}(n, d) = \{ \mU \in \R^{n \times d}, \mU^\top\mU = \mI_d \}$ is the orthogonal Stiefel manifold. This constraint prevents the embeddings from collapsing to $\bm{0}$.

\paragraph{Neighbor Embedding.} Another popular class of methods is the neighbor embedding framework. The central idea is to minimize the Kullback-Leibler divergence between two kernels $\mK_X$ and $\mK_Z$. 
\begin{align}\label{eq:SNE}
	\min_{\mZ \in \R^{n \times d}} \mathcal{J}_{L_{\mathrm{KL}}}(\mK_X, \mK_Z)
	\tag{NE}
\end{align}
Although some methods leave the kernels unnormalized (\textit{e.g.}\ UMAP by \cite{mcinnes2018umap}), the latter are usually taken as either row-stochastic (\textit{e.g.}\ SNE by \cite{hinton2002stochastic} and t-SNE by \cite{van2008visualizing}) or doubly-stochastic normalized (SNEkhorn by \cite{van2023snekhorn}). We briefly detail the latter as we rely on it in our experiments in \cref{sec:exps}. It consists in controlling the entropy in each point by solving the following OT problem
\begin{align}
	\min_{\mP \in \mathcal{H}_{\xi} \cap \mathcal{S}} \: \langle \mP, \mC \rangle \:. 
\end{align}
with $\mathcal{H}_\xi \coloneqq \{\mP \in \mathbb{R}_+^{n \times n} \ \text{s.t.} \ \mP \bm{1} = \bm{1} \: \ \text{and} \  \forall i, \: \operatorname{H}(\mP_{i:}) \geq \log{\xi} + 1 \}$ where the entropy of $\vp \in \mathbb{R}^{n}_+$ is\footnote{With the convention $0 \log 0 = 0$.} $\operatorname{H}(\vp) = -\sum_{i} p_i(\log(p_i)-1) = -\langle \vp, \log \vp - \bm{1} \rangle$. Note that at the optimum the entropy constraint is saturated thus allowing to accommodate for potentially varying noise levels while producing a doubly stochastic symmetric affinity matrix.

\section{(Semi-relaxed) Gromov-Wasserstein barycenter as a concave OT problem}\label{sec:srGW_concavity}

We consider here any graph $\mathcal{G}=(\mC, \vh)$ modeled as a connectivity matrix $\mC \in \R^{N \times N}$ and $\overline{\mC} \in \R^{n \times n}$ and a probability vector $\vh \in \Sigma_N^*$ and $\overline{\vh} \in \Sigma_n^*$. We focus next on the semi-relaxed Gromov-Wasserstein barycenter problem with an euclidean inner cost ($L = L_2$) reading as follows
\begin{equation}\label{eq:srgw_bary_supp} \tag{srGW-bary1}
	\min_{\overline{C} \in \R^{n \times n}} \srGW(\mC, \vh, \overline{\mC}) \quad \Leftrightarrow \quad \min_{\overline{\mC} \in \R^{n \times n}} \min_{\mT \in \mathcal{U}_n(\vh)} \mathcal{E}_2(\mC, \overline{\mC}, \mT)
\end{equation}
where $\mathcal{E}_2$ coincides with the objective function in \eqref{eq:gw_pb} applied in $\mC$ and $\overline{\mC}$. In general, the latter is considered as a non-convex problem. Notice that the subproblem w.r.t $\overline{\mC}$ is convex. While the subproblem w.r.t $\mT$ is in general non-convex and is equivalent to a quadractic program with Hessian matrix $\mathcal{H} = \overline{\mC}^2 \otimes_K \bm{1}_N \bm{1}_N^\top - 2 \overline{\mC} \otimes_K \mC$, where $\otimes_K$ is the kronecker product and the power operation is taken element-wise.

In the following, we proof a sufficient condition so that membership matrices are optimal for the \ref{eq:srgw_bary_supp} problem stated as such: 

\setcounter{theorem}{0}
\begin{theorem}\label{theo:srgw_bary_concavity_supp}
	Let $\mC \in \R^{N \times N}$ any bounded matrix and $\vh \in \Sigma_N^*$. Every solutions to the following problem
				\begin{equation}\label{eq:srgw_bary_equivalent_supp} \tag{srGW-bary2}
		\min_{\mT \in \mathcal{U}_n(\vh)} \mathcal{E}_2(\mC, \widetilde{\mC}(\mT), \mT) \: \text{ with } \: \forall (i, j) \in \integ{n}^2 \:, \: \overline{\mC}(\mT)_{ij} = \begin{cases} \left({\mT}^\top \mC \mT \oslash \overline{\vh} \: \overline{\vh}^\top \right)_{ij} & \text{ if  } \overline{h}_i \overline{h}_j  > 0 \\ \: 0 \: \text{ otherwise.}\end{cases}
	\end{equation}
	and $\overline{\vh}= \mT^\top \bm{1}_N $ are solutions to the \ref{eq:srgw_bary_supp} problem. Moreover If the function $g_{\mC}$ defined for any $\mU \in \mathcal{U}(\vh, \vh)$ as
	\begin{equation}
		g_{\mC}(\mU) = \mathrm{vec}(\mU)^\top \left(\mC \otimes_K \mC\right) \mathrm{vec}(\mU)
	\end{equation}
	is \underline{convex} on $\mathcal{U}(\vh, \vh)$, then the \ref{eq:srgw_bary_equivalent_supp} problem is \underline{concave} on $\mathcal{U}_n(\vh)$, hence problem \ref{eq:srgw_bary_supp} admits extremities of $\mathcal{U}_n(\vh)$ as OT solutions.
\end{theorem}
To prove Theorem \ref{theo:srgw_bary_concavity_supp}, let us begin with proving the following Lemma:
\setcounter{lemma}{0}
\begin{lemma}\label{lemma:srgw_bary_equivalentproblems}
	For any bounded matrix $\mC \in \R^{N \times N}$ and probability vector $\vh \in \Sigma_N^*$, every solutions to problem \ref{eq:srgw_bary_equivalent_supp} are solutions to problem \ref{eq:srgw_bary_supp}.
\end{lemma}
\begin{proof}[Proof of \Cref{lemma:srgw_bary_equivalentproblems}]

	Let us first characterize solutions to the \ref{eq:srgw_bary_supp} problem. Let $\mT \in \mathcal{U}_n(\vh)$. We will find a minimizer of the convex function $\overline{\mC} \in \R^{n \times n} \mapsto \mathcal{E}_2(\mC, \overline{\mC}, \mT)$ that we will denote by $\overline{\mC}$. Using the first order conditions and the convexity of this function, $\overline{\mC}$ is a solution \emph{if and only if} it satisfies
		\begin{equation}\label{eq:srgw_barycenter_FOC_Cbar}
			\nabla_{\overline{\mC}} \mathcal{E}_2(\mC, \overline{\mC}, \mT) = 2 \overline{\mC} \odot \overline{\vh} \: \overline{\vh}^{\top} - 2 \mT^{\top} \mC \mT = \bm{0}\,,
		\end{equation}
		where $\overline{\vh}$ depends on $\mT$ and is defined as $\overline{\vh}  = \mT^{\top} \bm{1}_N = \left(\|\mT_{:,1}\|_1, \cdots, \|\mT_{:,n}\|_1\right)^\top$ and $\mT_{:,j} \in \R^{N}$ denotes the $j$th column of $\mT$. We define
		\begin{equation}
			\label{eq:Cbar_FOC}
			\forall (i,j) \in \integ{n}^2, \ \overline{\mC}(\mT)_{ij} = \begin{cases} \left({\mT}^\top \mC \mT \oslash \overline{\vh} \: \overline{\vh}^\top \right)_{ij} & \text{ if  } \overline{h}_i \overline{h}_j = \|\mT_{:,i}\|_1 \|\mT_{:,j}\|_1 > 0 \\ 0 \text{ otherwise.}\end{cases}
																	\end{equation}
		On one hand for $(i,j) \in\integ{n}^2$ such that 		$\overline{h}_i \overline{h}_j >0, \ \overline{\mC}(\mT)_{ij} = \left({\mT}^\top \mC \mT \oslash \overline{\vh} \: \overline{\vh}^\top \right)_{ij}$ which clearly satisfies the first order conditions \cref{eq:srgw_barycenter_FOC_Cbar}. On the other hand, for $(i,j)\in \integ{n}^2$ such that 		$\overline{h}_i \overline{h}_j = 0$, we have $\mT_{:,i}= \mathbf{0}$ or $\mT_{:,j} = \mathbf{0}$, as $\forall i,j, \: T_{ij} \geq 0$. Hence
								\begin{equation}
			(2 \overline{\mC}(\mT) \odot \overline{\vh} \: \overline{\vh}^\top - 2 \mT^{\top}  \mC \mT)_{ij} = (- 2 \mT^{\top} \mC \mT)_{ij} = -2 \mT_{:,i}^\top  \mC \mT_{:,j}  = 0
		\end{equation}
		Overall, $\overline{\mC}(\mT)$ satisfies the first order conditions and thus is minimizing $\overline{\mC} \in \R^{n \times n} \mapsto \mathcal{E}_2(\mC, \overline{\mC}, \mT)$. Consequently solutions to \cref{eq:srgw_bary_supp} can be found by minimizing \begin{equation}\label{eq:def_F_srgw2}
			\mathcal{F}: \mT \in \mathcal{U}_n(\vh) \mapsto \mathcal{E}_2(\mC, \overline{\mC}(\mT), \mT) = \sum_{ijkl} | C_{ij}- \overline{\mC}(\mT)_{kl}|^2 T_{ik} T_{jl}
		\end{equation} 
		In order to prove the existence of a minimizer of $\mathcal{F}$ we will show that it is continuous on $\mathcal{U}_n(\vh)$ and conclude by compactness of $\mathcal{U}_n(\vh)$. 
																													For any $\mT \in \mathcal{U}_n(\vh)$, we have
		\begin{equation}
						\mathcal{F}(\mT) = \sum_{ij} C_{ij}^2 h_i h_j + \sum_{kl}  \overline{\mC}(\mT)_{kl}^2 \overline{h}_k \overline{h}_l- 2 \sum_{ijkl} C_{ij} \overline{\mC}(\mT)_{kl} T_{ik} T_{jl}
		\end{equation}
		 Now by definition of $\overline{\mC}(\mT)$ it satisfies the first order conditions \cref{eq:srgw_barycenter_FOC_Cbar} and in particular
		\begin{equation}
						\forall (k,l) \in \integ{n}^2,  \ \overline{\mC}(\mT)_{ij} \overline{h}_k \overline{h}_l= \mT_{:,k}^\top  \mC \mT_{:,l}\,.
		\end{equation}
		Thus $\sum_{kl}  \overline{\mC}(\mT)_{kl}^2 \overline{h}_k \overline{h}_l = \sum_{kl}  \overline{\mC}(\mT)_{kl} \mT_{:,k}^\top  \mC \mT_{:,l} =  \sum_{ijkl} C_{ij} \overline{\mC}(\mT)_{kl} T_{ik} T_{jl}$. Consequently the two last terms of $\mathcal{F}(\mT)$ simplify and we can reformulate
		\begin{equation}\label{eq:def_F_factorized}
			\begin{split}
				\mathcal{F}(\mT) &= \sum_{ij} C_{ij}^2 h_i h_j -\sum_{kl}  \overline{\mC}(\mT)_{kl} \mT_{:,k}^\top  \mC \mT_{:,l} \\
								&= \sum_{ij} C_{ij}^2 h_i h_j - \sum_{\begin{smallmatrix} k: \overline{h}_k \neq 0 \\ l: \overline{h}_l \neq 0 \end{smallmatrix}}  \frac{(\mT_{:,k}^\top  \mC \mT_{:,l})^{2}}{\overline{h}_k \overline{h}_l}
			\end{split}
		\end{equation}
		which is continuous on $\mathcal{U}_n(\vh)$.
		
		Therefore we ensured that \ref{eq:srgw_bary_supp} admits solutions of the form $(\mT, \overline{\mC}(\mT))$ where $\overline{\mC}(\mT)$ satisfies \eqref{eq:Cbar_FOC}. Moreover, these solutions can be found by minimizing w.r.t $\mT \in \mathcal{U}_n(\vh)$ the function $\mathcal{F}$  defined in \eqref{eq:def_F_srgw2}, which coincides with the problem \ref{eq:srgw_bary_equivalent_supp}.

\end{proof}

\paragraph{Concavity analysis.} The proof of Theorem \ref{theo:srgw_bary_concavity_supp} consists in studying the concavity on $\mathcal{U}_n(\vh)$ of the objective function $\mathcal{F}$ involved in problem
\ref{eq:srgw_bary_equivalent_supp}. To this end, we will prove that $\mathcal{F}$ is above its tangents. However,  we can see from \eqref{eq:def_F_factorized} that $\mathcal{F}$ is only differentiable on $\mathcal{U}_n(\vh) \backslash \overset{\circ}{\mathcal{U}}_n(\vh) $, where
\begin{equation}
\overset{\circ}{\mathcal{U}}_n(\vh) := \left\{\mT \in \mathcal{U}_n(\vh) \: | \: \mT^\top \bm{1}_N = \overline{\vh} > \bm{0}_n\right\}
\end{equation}
 is a convex subset of $\mathcal{U}_n(\vh)$. As $\mathcal{F}$ reads as a sum of rational functions whose respective denominator $\overline{h}_k \overline{h}_l = 0$ if and only if $\overline{h}_k = 0$ or $\overline{h}_l = 0$. Then we will first study the concavity of $\mathcal{F}$ on $\overset{\circ}{\mathcal{U}}_n(\vh)$. Then we will conclude on the concavity of $\mathcal{F}$ on $\mathcal{U}_n(\vh)$ by an argument of continuity.
Notice that the concavity of $\mathcal{F}$ on $\mathcal{U}_n(\vh)$ is equivalent to the convexity of the function 
\begin{equation} \label{eq:def_f}
	f: \mT \in \mathcal{U}_n(\vh) \mapsto  \sum_{\begin{smallmatrix} k: \overline{h}_k \neq 0 \\ l: \overline{h}_l \neq 0 \end{smallmatrix}}  \frac{(\mT_{:,k}^\top  \mC \mT_{:,l})^{2}}{\overline{h}_k \overline{h}_l}
\end{equation}
 which we will use next for the sake of simplicity. We start by emphasizing in the following lemma a low-rank factorization of $f$ which explicits its link with a GW problem from a graph to itself:
\begin{lemma}\label{lemma:gw_lowrank}
$\mathcal{F}$ admits as an equivalent low-rank formulation
\begin{equation}\label{eq:gw_lowrank}
	\mU \in \mathcal{V}_n(\vh) \rightarrow g_{\mC}(\mU) :=  \mathrm{vec}(\mU)^\top \left(\mC \otimes_K \mC\right) \mathrm{vec}(\mU)
\end{equation}
where $\mathcal{V}_n(\vh) := \left\{\mU \in \R^{N \times N} | \exists \mT \in \overset{\circ}{\mathcal{U}}_n(\vh) \text{  s.t  } \mT^\top \bm{1}_N = \overline{\vh}, \text{  }\mU = \mT diag(\overline{\vh}^{-1}) \mT^\top \right\} \subset \mathcal{U}(\vh, \vh)$.
\end{lemma}
\begin{proof}[Proof of Lemma \ref{lemma:gw_lowrank}]
For any $\mT \in \overset{\circ}{\mathcal{U}}_n(\vh)$, $f$ can be expressed as 
\begin{equation}
	\begin{split}
		f(\mT)&=   \| \mD_{\overline{\vh}}^{-1/2} \mT^{\top }\mC \mT \mD_{\overline{\vh}}^{-1/2} \|_F^2 \\
		&= \Tr \left\{ \mT \mD_{\overline{\vh}}^{-1} \mT^{\top }\mC^\top \mT \mD_{\overline{\vh}}^{-1} \mT^{\top }\mC \right\}\\
		(\text{posing } \mU = \mT \mD_{\overline{\vh}}^{-1} \mT^{\top }) \quad & =  \Tr \left\{ \mU \mC^\top \mU \mC \right\}\\
		&= \mathrm{vec}(\mU^\top)^\top \left(\mC \otimes_K \mC\right) \mathrm{vec}(\mU) \\
		&= \mathrm{vec}(\mU)^\top \left(\mC \otimes_K \mC\right) \mathrm{vec}(\mU) 
		:= g_{\mC}(\mU)\\
	\end{split}
\end{equation}
where $\mathrm{vec}$ denotes the column stacking operator and $\otimes_K$ the kronecker product. Following e.g \cite[equation 6]{vayer2018optimal}, one can see that $g_{\mC}$ relates to a low-rank Gromov-Wasserstein problem for a graph $\mC$ to itself, as $(\mU = \mT \mD_{\overline{\vh}}^{-1} \mT^{\top })$ is a coupling in $\mathcal{U}(\vh, \vh)$, resulting from the "self-gluing" of $\mT \in \overset{\circ}{\mathcal{U}}_n(\vh)$ where $rank(\mT) \leq n$. 
\end{proof}
Then we establish the following result
\begin{lemma}\label{lemma:Fconcavity_interior}
If the function $g_{\mC}(\mU) = \mathrm{vec}(\mU)^\top (\mC \otimes_K \mC) \mathrm{vec}(\mU)$ is convex on $\mathcal{U}(\vh, \vh)$, then $\mathcal{F}$ is concave on $\overset{\circ}{\mathcal{U}}_n(\vh)$.
\end{lemma}
\begin{proof}[Proof of Lemma \ref{lemma:Fconcavity_interior}] To establish the concavity of $\mathcal{F}$ on $\overset{\circ}{\mathcal{U}}_n(\vh)$, it suffices to prove that the function $f$ defined in equation \ref{eq:def_f} is convex on this set. To this end, as $f$ is in $\mathcal{C}^{1}(\overset{\circ}{\mathcal{U}}_n(\vh), \R_+)$, we will prove that it is above its tangents.
For any $(a,b) \in \integ{N} \times \integ{n}$, its first partial derivates read as
\begin{equation} \label{eq:Fgw_cost_gradient}
	\begin{split}
		&\frac{\partial}{\partial T_{ab}}f(\mT) \\
		&= \sum_{ij} \left\{  2 \left[\frac{\partial}{\partial T_{ab}}  \mT_{:, i}^\top \mC \mT_{:, j} \right]  \mT_{:, i}^\top \mC \mT_{:, j}  \frac{1}{\overline{h}_i \overline{h}_j} + \left( \mT_{:, i}^\top \mC \mT_{:, j}  \right)^2 \left[\frac{\partial}{\partial T_{ab}}  \frac{1}{\overline{h}_i \overline{h}_j}   \right]  \right\} \\
		&=  2 \sum_j  \scalar{\mC_{a, :}}{\mT_{:, j}}\mT_{:, b}^\top \mC \mT_{:, j}  \frac{1}{\overline{h}_b \overline{h}_j} + 2 \sum_i \scalar{\mC_{:, a}}{\mT_{:, i}}\mT_{:, i}^\top \mC \mT_{:, b}  \frac{1}{\overline{h}_i \overline{h}_b}\\
		&- \sum_{ij} \left( \mT_{:, i}^\top \mC \mT_{:, j}  \right)^2 \left\{ \frac{ \delta_{i=b}\overline{h}_j+ \delta_{j=b}\overline{h}_i}{\overline{h}_i^2 \overline{h}_j^2}\right\}  \\
										&=2 \sum_j  \scalar{\mC_{a, :}}{\mT_{:, j}}\mT_{:, b}^\top \mC \mT_{:, j}  \frac{1}{\overline{h}_b \overline{h}_j} + 2 \sum_i \scalar{\mC_{:, a}}{\mT_{:, i}}\mT_{:, i}^\top \mC \mT_{:, b}  \frac{1}{\overline{h}_i \overline{h}_b} \\
		&- \sum_{j} \left( \mT_{:, b}^\top \mC \mT_{:, j}  \right)^2\frac{1}{\overline{h}_b^2 \overline{h}_j} -  \sum_{i} \left( \mT_{:, i}^\top \mC \mT_{:, b}  \right)^2\frac{1}{\overline{h}_i \overline{h}_b^2} \\
	\end{split}
\end{equation}
Consider now any admissible couplings $\mT^{(1)} \in \overset{\circ}{\mathcal{U}}_n(\vh)$ and $\mT^{(2)} \in \overset{\circ}{\mathcal{U}}_n(\vh)$, we want to prove that
\begin{equation}\label{eq:tangent_inequality}
	f(\mT^{(1)}) \geq 	f(\mT^{(2)}) + \scalar{\nabla_{\mT}f(\mT^{(2)})}{\mT^{(1)} - \mT^{(2)}}_F 
\end{equation}
First observe that we have
																										\begin{equation}
	\begin{split}
		&\scalar{\nabla_{\mT}f(\mT^{(2)})}{\mT^{(1)} - \mT^{(2)}}_F\\ 
		&= \sum_{ab} (T_{ab}^{(1)} - T_{ab}^{(2)})
		\left\{
		2 \sum_j  \scalar{\mC_{a, :}}{\mT^{(2)}_{:, j}}\mT_{:, b}^{(2)\top} \mC \mT^{(2)}_{:, j}  \frac{1}{\overline{h}^{(2)}_b \overline{h}^{(2)}_j} + 2 \sum_j \scalar{\mC_{:, a}}{\mT^{(2)}_{:, j}}\mT_{:, j}^{(2)\top} \mC \mT^{(2)}_{:, b}  \frac{1}{\overline{h}^{(2)}_j \overline{h}^{(2)}_b} \right\}\\
		&-  \sum_{ab} (T_{ab}^{(1)} - T_{ab}^{(2)}) \left\{  \sum_{j} \left( \mT_{:, b}^{(2)\top} \mC \mT^{(2)}_{:, j}  \right)^2\frac{1}{\overline{h}_b^{(2)2} \overline{h}^{(2)}_j} +  \sum_{j} \left( \mT_{:, j}^{(2)\top} \mC \mT^{(2)}_{:, b}  \right)^2\frac{1}{\overline{h}^{(2)}_j \overline{h}_b^{(2)2}} \right\}\\				
				&=  
		2 \sum_{bj} \left(  \mT^{(1) \top }_{:, b}\mC\mT^{(2)}_{:, j} \right) \left( \mT_{:, b}^{(2)\top} \mC \mT^{(2)}_{:, j} \right) \frac{1}{\overline{h}^{(2)}_b \overline{h}^{(2)}_j} + 2 \sum_{bj} \left( \mT^{(1) \top }_{:, b}\mC^\top \mT^{(2)}_{:, j}\right) \left( \mT_{:, j}^{(2)\top} \mC \mT^{(2)}_{:, b}\right)  \frac{1}{\overline{h}^{(2)}_j \overline{h}^{(2)}_b} \\
		&-2 \sum_{bj}  \left(  \mT^{(2) \top }_{:, b}\mC\mT^{(2)}_{:, j} \right)^2 \frac{1}{\overline{h}^{(2)}_b \overline{h}^{(2)}_j} - 2 \sum_{bj} \left( \mT_{:, j}^{(2)\top} \mC \mT^{(2)}_{:, b} \right)^2  \frac{1}{\overline{h}^{(2)}_j \overline{h}^{(2)}_b}\\ 
		&-  \sum_{bj} \left\{  \left( \mT_{:, b}^{(2)\top} \mC \mT^{(2)}_{:, j}  \right)^2 +  \left( \mT_{:, j}^{(2)\top} \mC \mT^{(2)}_{:, b}  \right)^2\right\} \left\{  \frac{\overline{h}_b^{(1)} }{\overline{h}_b^{(2)2} \overline{h}^{(2)}_j} - \frac{1}{\overline{h}_b^{(2)} \overline{h}^{(2)}_j} \right\}  \\
		&=  
		2 \sum_{bj}  \mT^{(1) \top }_{:, b}\mC\mT^{(2)}_{:, j}\mT_{:, b}^{(2)\top} \mC \mT^{(2)}_{:, j}  \frac{1}{\overline{h}^{(2)}_b \overline{h}^{(2)}_j} + 2 \sum_{bj} \mT^{(1) \top }_{:, b}\mC^\top \mT^{(2)}_{:, j}\mT_{:, j}^{(2)\top} \mC \mT^{(2)}_{:, b}  \frac{1}{\overline{h}^{(2)}_j \overline{h}^{(2)}_b} \\
		&-  \sum_{bj} \left\{  \left( \mT_{:, b}^{(2)\top} \mC \mT^{(2)}_{:, j}  \right)^2 +  \left( \mT_{:, j}^{(2)\top} \mC \mT^{(2)}_{:, b}  \right)^2\right\}\frac{\overline{h}_b^{(1)} }{\overline{h}_b^{(2)2} \overline{h}^{(2)}_j} -2 \mathcal{F}(\mT^{(2)}) \\
	\end{split}
\end{equation}
So the difference of both terms in \eqref{eq:tangent_inequality} reads:\newline
\begin{equation}\label{eq:gw_abovetangents1}
	\begin{split}
		&f(\mT^{(1)}) - 	f(\mT^{(2)}) - \scalar{\nabla_{\mT}f(\mT^{(2)})}{\mT^{(1)} - \mT^{(2)}}_F \\
		&=\sum_{bj} \left( \mT_{:, b}^{(1)\top} \mC \mT^{(1)}_{:, j}  \right)^2 \frac{1}{\overline{h}^{(1)}_b \overline{h}^{(1)}_j} + \sum_{bj} \left( \mT_{:, b}^{(2)\top} \mC \mT^{(2)}_{:, j}  \right)^2 \frac{1}{\overline{h}^{(2)}_b \overline{h}^{(2)}_j} \\
		&+   \sum_{bj} \left\{  \left( \mT_{:, b}^{(2)\top} \mC \mT^{(2)}_{:, j}  \right)^2 +  \left( \mT_{:, b}^{(2)\top} \mC^\top \mT^{(2)}_{:, j}  \right)^2\right\}\frac{\overline{h}_b^{(1)} }{\overline{h}_b^{(2)2} \overline{h}^{(2)}_j} \\
		&-2 \sum_{bj} \left(  \mT^{(1) \top }_{:, b}\mC\mT^{(2)}_{:, j} \right) \left( \mT_{:, b}^{(2)\top} \mC \mT^{(2)}_{:, j}  \right)\frac{1}{\overline{h}^{(2)}_b \overline{h}^{(2)}_j} - 2 \sum_{bj} \left(  \mT^{(1) \top }_{:, b}\mC^\top \mT^{(2)}_{:, j} \right) \left( \mT_{:, b}^{(2)\top} \mC^\top \mT^{(2)}_{:, j} \right) \frac{1}{\overline{h}^{(2)}_j \overline{h}^{(2)}_b} \\
	\end{split}
\end{equation}
Then notice that for any $(b, j)$, we have
\begin{equation}
	\begin{split}
		& \left(\mT^{(1) \top}_{:,b} \mC \mT^{(2)}_{:, j} \right)^2\frac{1}{\overline{h}^{(1)}_b\overline{h}^{(2)}_j} +   \left( \mT_{:, b}^{(2)\top} \mC \mT^{(2)}_{:, j}  \right)^2 \frac{ \overline{h}^{(1)}_b }{\overline{h}^{(2)2}_b \overline{h}^{(2)}_j} 
		- 2 \left(\mT_{:, b}^{(1)\top} \mC \mT^{(2)}_{:, j}\right)\left(  \mT_{:, b}^{(2)\top} \mC \mT^{(2)}_{:, j} \right) \frac{1}{\overline{h}^{(2)}_b \overline{h}^{(2)}_j}   \\
		&= \left( \frac{\mT^{(1) \top}_{:,b} \mC \mT^{(2)}_{:, j}}{\sqrt{\overline{h}^{(1)}_b\overline{h}^{(2)}_j}} -  \mT_{:, b}^{(2)\top} \mC \mT^{(2)}_{:, j}  \sqrt{\frac{ \overline{h}^{(1)}_b }{\overline{h}^{(2)2}_b \overline{h}^{(2)}_j} }
		\right)^2 = (A_{bj} - B_{bj})^2\\
	\end{split}
\end{equation}
then similarly we have
\begin{equation}
	\begin{split}
		& \left(\mT^{(1) \top}_{:,b} \mC^\top \mT^{(2)}_{:, j} \right)^2\frac{1}{\overline{h}^{(1)}_b\overline{h}^{(2)}_j} +   \left( \mT_{:, b}^{(2)\top} \mC^\top \mT^{(2)}_{:, j}  \right)^2 \frac{ \overline{h}^{(1)}_b }{\overline{h}^{(2)2}_b \overline{h}^{(2)}_j} 
		- 2  \left(  \mT^{(1) \top }_{:, b}\mC^\top \mT^{(2)}_{:, j} \right) \left( \mT_{:, b}^{(2)\top} \mC^\top \mT^{(2)}_{:, j} \right) \frac{1}{\overline{h}^{(2)}_b \overline{h}^{(2)}_j}    \\
		&= \left( \frac{\mT^{(1) \top}_{:,b} \mC^\top \mT^{(2)}_{:, j}}{\sqrt{\overline{h}^{(1)}_b\overline{h}^{(2)}_j}} -  \mT_{:, b}^{(2)\top} \mC^\top \mT^{(2)}_{:, j}  \sqrt{\frac{ \overline{h}^{(1)}_b }{\overline{h}^{(2)2}_b \overline{h}^{(2)}_j} }
		\right)^2 = (A'_{bj} - B'_{bj})^2\\
	\end{split}
\end{equation}
So we can express the \eqref{eq:gw_abovetangents1} as 
\begin{equation}\label{eq:gw_abovetangents2}
	\begin{split}
		&f(\mT^{(1)}) - 	f(\mT^{(2)}) - \scalar{\nabla_{\mT}\mathcal{F}^{GW}(\mT^{(2)})}{\mT^{(1)} - \mT^{(2)}}_F \\
		&= \sum_{bj} \left( \mT_{:, b}^{(1)\top} \mC \mT^{(1)}_{:, j}  \right)^2 \frac{1}{\overline{h}^{(1)}_b \overline{h}^{(1)}_j} 
		+ \sum_{bj} \left( \mT_{:, b}^{(2)\top} \mC \mT^{(2)}_{:, j}  \right)^2 \frac{1 }{\overline{h}^{(2)}_b \overline{h}^{(2)}_j} \\
		&- \sum_{bj}\left\{  \left(\mT^{(1) \top}_{:,b} \mC \mT^{(2)}_{:, j} \right)^2 + \left(\mT^{(1) \top}_{:,b} \mC^\top \mT^{(2)}_{:, j} \right)^2 \right\} \frac{1}{\overline{h}^{(1)}_b\overline{h}^{(2)}_j} + \sum_{bj} \left\{ (A_{bj} - B_{bj})^2 + (A'_{bj} - B'_{bj})^2\right\}
	\end{split}
\end{equation}
Now let us suppose that the function $g_{\mC}$ defined in \eqref{eq:gw_lowrank} of Lemma \ref{lemma:gw_lowrank} is convex on $\mathcal{U}(\vh, \vh)$, hence including low-rank couplings of the form $\mU = \mT \mD_{\overline{\vh}}^{-1} \mT^{\top }$. Given $\mU^{(1)}= \mT^{(1)} \mD_{\overline{\vh}^{(1)}}^{-1}\mT^{(1)\top}$ and $\mU^{(2)}= \mT^{(2)} \mD_{\overline{\vh}^{(2)}}^{-1}\mT^{(2)\top}$, the convexity of $g_{\mC}$ implies that for any $\lambda \in \left[0, 1\right]$,
\begin{equation}
	\begin{split}
		&g_{\mC}(\lambda \mU^{(1)} + (1-\lambda) \mU^{(2)}) \\
		&= \Tr \left\{ \left( \lambda \mU^{(1)} + (1-\lambda) \mU^{(2)} \right) \mC^\top \left( \lambda \mU^{(1)} + (1-\lambda) \mU^{(2)} \right) \mC \right\}\\
		&= \lambda^2\Tr \left\{\mU^{(1)}  \mC^\top \mU^{(1)} \mC \right\} +(1-\lambda)^2 \Tr \left\{ \mU^{(2)}\mC^\top \mU^{(2)} \mC \right\}  + 2 \lambda(1-\lambda) \Tr \left\{  \mU^{(1)} \mC^\top \mU^{(2)}\mC \right\}\\ 
		& \leq \lambda \Tr \left\{\mU^{(1)}  \mC^\top \mU^{(1)} \mC \right\} +(1-\lambda) \Tr \left\{ \mU^{(2)}\mC^\top \mU^{(2)} \mC \right\} 
	\end{split}
\end{equation}
implying e.g for $\lambda = \frac{1}{2}$, that
\begin{equation}\label{eq:gw_lowrank_convexity}
	\begin{split}
		\Tr \left\{\mU^{(1)}  \mC^\top \mU^{(1)} \mC \right\}
		& \leq \frac{1 }{2}\Tr \left\{\mU^{(1)} \mC^\top  \mU^{(1)} \mC \right\} + \frac{1 }{2} \Tr \left\{ \mU^{(2)}\mC^\top\mU^{(2)}  \mC \right\} 
	\end{split}
\end{equation}
where for instance
\begin{equation}
	\begin{split}
		\Tr \left\{  \mU^{(1)} \mC^\top \mU^{(2)} \mC \right\} &= 	\Tr \left\{  \mT^{(1)} \mD_{\overline{\vh}^{(1)}}^{-1} \mT^{{(1)}\top } \mC^\top \mT^{(2)} \mD_{\overline{\vh}^{(2)}}^{-1} \mT^{{(2)}\top } \mC \right\} \\
		&= 	\Tr \left\{ \mD_{\overline{\vh}^{(2)}}^{-1/2} \mT^{{(2)}\top } \mC^\top  \mT^{(1)} \mD_{\overline{\vh}^{(1)}}^{-1/2} \mD_{\overline{\vh}^{(1)}}^{-1/2} \mT^{{(1)}\top } \mC \mT^{(2)} \mD_{\overline{\vh}^{(2)}}^{-1} \right\} \\
		&= 	\| \mD_{\overline{\vh}^{(1)}}^{-1/2} \mT^{{(1)}\top } \mC \mT^{(2)} \mD_{\overline{\vh}^{(2)}}^{-1} \|_F^2\\
		&= \sum_{ij} \left( \mT^{(1)\top}_{:, i} \mC \mT^{(2)}_{:, j}  \right)^2 \frac{1}{\overline{h}^{(1)}_i \overline{h}^{(2)}_j}\\
		&= \sum_{ij} \left( \mT^{(2)\top}_{:, j} \mC^\top \mT^{(1)}_{:, i}  \right)^2 \frac{1}{\overline{h}^{(1)}_i \overline{h}^{(2)}_j}\\
	\end{split}
\end{equation}
The last equality holds as $\mT^{(1)\top}_{:, i} \mC \mT^{(2)}_{:, j}  = \mT^{(2)\top}_{:, j} \mC^\top \mT^{(1)}_{:, i} $. Notice that using the same kind of relations we have
\begin{equation}\label{eq:cost_sym_relations}
	\sum_{bj}\frac{\left( \mT_{:, b}^{(1)\top} \mC \mT^{(1)}_{:, j}  \right)^2}{\overline{h}^{(1)}_b \overline{h}^{(1)}_j} = 	\sum_{bj} \frac{\left( \mT_{:, j}^{(1)\top} \mC^\top \mT^{(1)}_{:, b}  \right)^2}{\overline{h}^{(1)}_b \overline{h}^{(1)}_j} = 	\sum_{bj}\frac{\left( \mT_{:, b}^{(1)\top} \mC^\top \mT^{(1)}_{:, j}  \right)^2}{\overline{h}^{(1)}_b \overline{h}^{(1)}_j}  
\end{equation}
This way we can express the concavity inequality in \eqref{eq:gw_lowrank_convexity} as follows 
\begin{equation}\label{eq:gw_lowrank_convexity_C}
	\sum_{bj} \frac{ \left( \mT_{:, b}^{(1)\top} \mC \mT^{(1)}_{:, j}  \right)^2}{\overline{h}^{(1)}_b \overline{h}^{(1)}_j}
	+ \sum_{bj}  \frac{ \left( \mT_{:, b}^{(2)\top} \mC \mT^{(2)}_{:, j}  \right)^2}{\overline{h}^{(2)}_b \overline{h}^{(2)}_j }  - 2 \sum_{bf} \frac{\left(\mT^{(1) \top}_{:,b} \mC \mT^{(2)}_{:, j} \right)^2 }{\overline{h}^{(1)}_b\overline{h}^{(2)}_j} \geq 0
\end{equation}
and symetrically using \eqref{eq:cost_sym_relations} as
\begin{equation}\label{eq:gw_lowrank_convexity_Ct}
	\sum_{bj} \frac{ \left( \mT_{:, b}^{(1)\top} \mC^\top \mT^{(1)}_{:, j}  \right)^2}{\overline{h}^{(1)}_b \overline{h}^{(1)}_j}
	+ \sum_{bj}  \frac{ \left( \mT_{:, b}^{(2)\top} \mC^\top \mT^{(2)}_{:, j}  \right)^2}{\overline{h}^{(2)}_b \overline{h}^{(2)}_j }  - 2 \sum_{bf} \frac{\left(\mT^{(1) \top}_{:,b} \mC^\top \mT^{(2)}_{:, j} \right)^2 }{\overline{h}^{(1)}_b\overline{h}^{(2)}_j} \geq 0
\end{equation}
So we can conclude from \eqref{eq:gw_abovetangents2}, \eqref{eq:gw_lowrank_convexity_C} and \eqref{eq:gw_lowrank_convexity_Ct} that 
\begin{equation}
	\begin{split}
		&f(\mT^{(1)}) - 	f(\mT^{(2)}) - \scalar{\nabla_{\mT}f(\mT^{(2)})}{\mT^{(1)} - \mT^{(2)}}_F \\
		&\geq  \sum_{bj} \left\{ (A_{bj} - B_{bj})^2 + (A'_{bj} - B'_{bj})^2\right\}\\
		& \geq 0
	\end{split}
\end{equation}
Hence it is enough to have $g_{\mC}$ convex on $\mathcal{U}(\vh, \vh)$ to get $f$ convex  on $\overset{\circ}{\mathcal{U}}_n(\vh)$, and equivalently $\mathcal{F}$ concave on $\overset{\circ}{\mathcal{U}}_n(\vh)$.
\end{proof}
\begin{proof}[\textbf{Proof of Theorem \ref{theo:srgw_bary_concavity_supp}}]
	Following Lemma \ref{lemma:Fconcavity_interior}, if $g_{\mC}$ is convex on $\mathcal{U}(\vh, \vh)$ we know that $\mathcal{F}$ is concave on  $\overset{\circ}{\mathcal{U}}_n(\vh)$. Moreover, we also proved in Lemma \ref{lemma:srgw_bary_equivalentproblems} that $\mathcal{F}$ is continuous on $\mathcal{U}_n(\vh)$.
	
	Now let us consider any $\mE \in \mathcal{U}_n(\vh) \backslash \overset{\circ}{\mathcal{U}}_n(\vh)$, \emph{i.e} there exists at least one $i \in \integ{n}$, such that $\| \mE_{:, i} \|_1 = 0$. Let any $\mT \in \overset{\circ}{\mathcal{U}}_n(\vh)$, and any $\lambda \in \left[0, 1\right]$. As $\mathcal{U}_n(\vh)$ is compact, we can  define a sequence $\left\{\mV^{(m)} = \frac{1}{m} \mT + (1-\frac{1}{m})\mE\right\}_{m \in \mathbb{N}}$ such that $\mV^{(m)} \xrightarrow[m \rightarrow \infty]{} \mE$. By construction, $\forall m, \: \mV^{(m)} \in \overset{\circ}{\mathcal{U}}_n(\vh)$, as $\mV^{(m)} \in \mathcal{U}_n(\vh)$ by convexity and $\forall i \in \integ{n}, \: \| \mV^{(m)}_{:, i} \|_1 = \frac{1}{m} \| \mT_{:, i} \|_1 + (1-\frac{1}{m})\| \mE_{:, i} \|_1 >0$. Then we have by concavity in $\overset{\circ}{\mathcal{U}}_n(\vh)$:
	\begin{equation}
		\mathcal{F}(\lambda \mT + (1-\lambda) \mV_m) \geq \lambda \mathcal{F}(\mT) + (1-\lambda) \mathcal{F}(\mV_m)
	\end{equation}
	then by continuity of $\mathcal{F}$ on $\mathcal{U}_n(\vh)$, we have when $m \rightarrow \infty$,
	\begin{equation}
		\mathcal{F}(\lambda \mT + (1-\lambda) \mE) \geq \lambda \mathcal{F}(\mT) + (1-\lambda) \mathcal{F}(E)
	\end{equation}
	which holds for any $\mT \in \overset{\circ}{\mathcal{U}}_n(\vh)$ and any $\lambda \in \left[0, 1\right]$. Notice that the same reasoning can be done for $\mT \in \mathcal{U}_n(\vh) \backslash \overset{\circ}{\mathcal{U}}_n(\vh)$ by considering another analog sequence that converges to $\mT$. So we might conclude that $\mathcal{F}$ is concave on $\mathcal{U}_n(\vh)$. Therefore problem \ref{eq:srgw_bary_equivalent_supp} is a concave problem over a polytope, hence admits extremities of $\mathcal{U}_n(\vh)$ as minimum, and so does \ref{eq:srgw_bary_supp} thanks to Lemma \ref{lemma:srgw_bary_equivalentproblems}. Notice that one can express extremities of $\mathcal{U}_n(\vh)$ as $\left\{ \text{diag}(\vh) \mM | \mM \in \left\{0, 1\right\}^{N \times n}, \: \forall i \in \integ{N}, \: \exists ! j \in \integ{n}, M_{ij} =1 \right\}$ \cite[Theorem 1]{cao2022centrosymmetric}.
\end{proof}
\paragraph{Extension to GW.} Finally for the sake of completeness, we can follow an analog development for $\GW$ instead of $\srGW$, i.e considering the barycenter distribution fixed to $\overline{\vh} \in \Sigma_n^*$, leading to the following GW barycenter problem:
\begin{equation}\label{eq:gw_problem1} \tag{GW-bary-1}
\begin{split}
	\min_{\overline{\mC} \in \R^{n \times n }, \mT \in \mathcal{U}(\vh, \overline{\vh})}  \mathcal{E}^{GW}(\mC, \overline{\mC}, \mT) 
\end{split}
\end{equation}
Usong the same notations than in Theorem \ref{theo:srgw_bary_concavity_supp}, we can state the next result:
\setcounter{corollary}{0}
\begin{corollary}\label{corr:gwbary_concavity}
Let $\mC \in \R^{N \times N}$ any bounded matrix, $\vh \in \Sigma_N^*$ and $\overline{\vh} \in \Sigma_n^*$. If the function $g_{\mC}$ defined for any $\mU \in \mathcal{U}(\vh, \vh)$ as
\begin{equation}
	g_{\mC}(\mU) = \mathrm{vec}(\mU)^\top \left(\mC \otimes_K \mC\right) \mathrm{vec}(\mU)
\end{equation}
is \underline{convex} on $\mathcal{U}(\vh, \vh)$, then the following \ref{eq:gw_problem2} problem
\begin{equation} \label{eq:gw_problem2} \tag{GW-bary-2}
	\min_{\mT \in \mathcal{U}(\vh, \overline{\vh})} \mathcal{E}^{GW}(\mC, \widetilde{\mC}(\mT), \mT) 
\end{equation}
is \underline{concave}. Hence the \ref{eq:gw_problem1} problem
admits extremities of $\mathcal{U}(\vh, \overline{\vh})$ as optimum.
\end{corollary}
\begin{proof}[Proof of Corollary \ref{corr:gwbary_concavity}]
Assuming that $g_{\mC}$ is convex on $\mathcal{U}(\vh, \vh)$, implies that the \ref{eq:srgw_bary_equivalent_supp} problem is concave as the objective function $\mT \rightarrow \mathcal{E}_2(\mC, \widetilde{\mC}(\mT), \mT)$ is concave on $\mathcal{U}_n(\vh)$. Therefore this function is necessarily concave on $\mathcal{U}(\vh, \overline{\vh})$ which is a convex subset of $\mathcal{U}_n(\vh)$. So we can conclude that the \ref{eq:gw_problem2} problem is concave.
\end{proof}

\section{Additional Details for Methods and Experiments}\label{sec:appendix_exps}
\subsection{Extension to the Fused Gromov-Wasserstein Framework}\label{sec:FGW_DR}
As mentioned in Section \ref{sec:model_GWDR}, we further propose to extend \ref{eq:GW_DR_pb} to the Fused Gromov-Wasserstein framework in order to explicitly incorporate features $\mX$. It reads as follows
\begin{equation}\label{eq:FGW_DR_pb}\tag{FGW-DR}
\min_{\mZ \in \R^{n \times d}, \overline{\mF} \in \R^{n \times p}} \mathrm{srFGW}_{\alpha, L}(\mC_X, \mX, \vh_X, \mC_Z, \overline{\mF}) \:.
\end{equation}
where $\srFGW_{\alpha, L}$ relates to the semi-relaxed Fused Gromov-Wasserstein divergence parametrized by  $\alpha \in \left[0, 1\right]$ and the choice of inner-loss for $L$ taken as $L_2$ or $L_{KL}$. Following notations in Section \ref{sec:srGW_concavity}, this divergence can be expressed as follows
\begin{equation}\label{eq:fgw_pb}
\min_{\mT \in \mathcal{U}_n(\vh_X)} \:  \alpha \: \sum_{ijkl}  L([\mC_X]_{ij}, [\mC_Z]_{kl}) T_{ik} T_{jl} \: + \: (1 - \alpha ) \: \sum_{ijk} L\left(X_{ik}, \overline{F}_{jk}\right) T_{ij}
\end{equation}
where $\mathcal{U}_n(\vh_X) = \left\{ \mT \in \R_+^{N \times n} | \mT \bm{1}_n = \vh_X \right\}$. As such, srFGW
aims at finding a (semi-relaxed) optimal coupling by minimizing an OT
cost which is a trade-off of a Wasserstein cost between feature matrices and a GW cost between the similarity matrices. 
As such, \ref{eq:FGW_DR_pb} comes down to regularizing the inner OT problem with a semi-relaxed Wasserstein barycenter problem. The latter essentially reduces to a concave problem, wherein the goal is to achieve K-means clustering on $\mX$ \cite{canas2012learning}. We acknowledge that the authors do not address this problem from an optimization point of view. To this end, one can follow an analog scheme than in the proof of Theorem \ref{theo:srgw_bary_concavity_supp} in the Wasserstein setting. Similarly, the minimization w.r.t $\overline{\mF}$ of the Wasserstein barycenter objective admits closed-form solutions given $\mT \in \mathcal{U}_n(\vh_X)$, denoted $\widetilde{F}(\mT)$ \cite{vayer2018optimal}. Problem \ref{eq:FGW_DR_pb} then can be equivalently written as 
\begin{equation}\label{eq:FGW_DR_pb_equi}
\min_{\mZ \in \R^{n \times d}, \mT \in \mathcal{U}_n(\vh_X)} \:  \alpha \: \sum_{ijkl}  L([\mC_X]_{ij}, [\mC_Z]_{kl}) T_{ik} T_{jl} \: + \: (1 - \alpha ) \: \sum_{ijk} L\left(X_{ik}, \widetilde{F}_{jk}(\mT)\right) T_{ij}
\end{equation}
where the second term relates to a concave function w.r.t $\mT$,  hence acting as a concave regularization w.r.t to $\mT$ of \ref{eq:GW_DR_pb}.

\subsection{Experiments}

\paragraph{Datasets.}
We first provide details about the datasets used in \cref{sec:exps}.

\begin{table}[h]
\centering
\caption{Dataset Details.}\label{tab:dataset_details}
\scalebox{0.9}{
	\begin{tabular}{|c||c|c|c|} \hline
		& Number of samples & Dimensionality & Number of classes \\ \hline \hline
		MNIST & $10000$ & $784$ & $10$ \\ \hline
		F-MNIST & $10000$ & $784$ & $10$ \\ \hline
		COIL & $1440$ & $16384$ & $20$ \\ \hline
\end{tabular}}
\end{table}

\paragraph{Grid search for fused GW.}
As the two terms appearing in fused GW \cite{vayer2018optimal} may have different scales, we have to test a quite wide spectrum of values. For \cref{tab:alpha_scores}, we use the following grid
\begin{align}
\{0, 0.000001, 0.0003, 0.005, 0.1, 0.25, 0.5, 0.75, 0.9, 0.995, 0.9997, 0.999999, 1\} \:.
\end{align}

\paragraph{About the implementation of GWDR.}
To initialize the prototypes' position, we sample independent $\mathcal{N}(0,1)$ coordinates. Similarly, we initialize the transport plans by sampling uniform random variables in $[0,1]$ before normalizing such that the marginal constraint is satisfied.

\end{document}